\documentclass[lettersize,journal]{IEEEtran}
\usepackage{amsmath,amsfonts}
\usepackage{algorithm}
\usepackage{array}
\usepackage[caption=false,font=normalsize,labelfont=sf,textfont=sf]{subfig}
\usepackage{textcomp}
\usepackage{stfloats}
\usepackage{url}
\usepackage{verbatim}
\usepackage{graphicx}
\usepackage{cite}
\newcommand{\ie}{\textit{i.e.}}
\newcommand{\eg}{\textmd{e.g.}}
\usepackage{times}
\usepackage{epsfig}
\usepackage{graphicx}
\usepackage{amsmath}
\usepackage{amssymb}
\usepackage{adjustbox}

\usepackage{tabularx}
\usepackage[english]{babel}
\usepackage{amsthm}
\usepackage{bbold}

\usepackage[pagebackref=true,breaklinks=true,letterpaper=true,colorlinks,bookmarks=false]{hyperref}

\usepackage{graphicx}
\usepackage{amsmath}
\usepackage{amssymb}
\usepackage{booktabs}
\usepackage{subcaption}
\usepackage{bbding}
\usepackage{pifont}
\usepackage{wasysym}
\usepackage{amssymb}
\usepackage{graphicx}
\usepackage{amsthm}
\usepackage{dsfont}
\usepackage{amsfonts}
\usepackage{bm}
\usepackage{threeparttable}
\usepackage{multirow}
\usepackage{booktabs}
\usepackage{algpseudocode}%

\newtheorem{definition}{Definition}[section]

\newtheorem{theorem}{Theorem}

\newtheorem{assumption}{Assumption}

\begin{document}

\title{INSURE: an Information theory iNspired diSentanglement and pURification\\modEl for domain generalization}

\author{Xi Yu, Huan-Hsin Tseng, Shinjae Yoo, Haibin Ling, Yuewei Lin*
\thanks{X. Yu, H.~H. Tseng, S. Yoo and Y. Lin are with Computational Science Initiative, Brookhaven National Laboratory, Upton, NY, USA.}
\thanks{H. Ling is with the Department of Computer Science, Stony Brook University, Stony Brook, NY, USA.}
\thanks{*Y. Lin is the corresponding author.}
}

\markboth{Journal of \LaTeX\ Class Files,~Vol.~14, No.~8, August~2021}%
{Shell \MakeLowercase{\textit{et al.}}: A Sample Article Using IEEEtran.cls for IEEE Journals}


\maketitle

\begin{abstract}
Domain Generalization (DG) aims to learn a generalizable model on the unseen target domain by only training on the multiple observed source domains. Although a variety of DG methods have focused on extracting domain-invariant features, the domain-specific class-relevant features have attracted attention and been argued to benefit generalization to the unseen target domain. To take into account the class-relevant domain-specific information, in this paper we propose an \textbf{I}nformation theory i\textbf{N}spired di\textbf{S}entanglement and p\textbf{UR}ification  mod\textbf{E}l (INSURE) to explicitly disentangle the latent features to obtain sufficient and compact (necessary) class-relevant feature for generalization to the unseen domain. Specifically, we first propose an information theory inspired loss function to ensure the disentangled class-relevant features contain sufficient class label information and the other disentangled auxiliary feature has sufficient domain information. We further propose a paired purification loss function to let the auxiliary feature discard all the class-relevant information and thus the class-relevant feature will contain sufficient and compact (necessary) class-relevant information. Moreover, instead of using multiple encoders, we propose to use a learnable binary mask as our disentangler to make the disentanglement more efficient and make the disentangled features complementary to each other. We conduct extensive experiments on four widely used DG benchmark datasets including PACS, OfficeHome, TerraIncognita, and DomainNet. The proposed INSURE outperforms the state-of-art methods. We also empirically show that domain-specific class-relevant features are beneficial for domain generalization. 
\end{abstract}


\section{Introduction}
\label{sec:intro}
A fundamental assumption in most statistical machine learning algorithms is that training data and test data have independent and identical distributions (i.i.d.). However, this assumption does not always hold in real-world applications due to the distribution shift between source data and target data. For instance, a car detector should adapt to different environments (\textit{\eg,} urban to rural shift) and adverse weather conditions (\textit{\eg,} sunny to rainy shift)~\cite{zhou2022domain}. The classic deep learning model often fails to generalize to test data under such out-of-distribution (OOD) scenarios since the i.i.d. assumption is violated.

To mitigate this domain shift problem, domain generalization (DG) was introduced in~\cite{blanchard2011generalizing}. In DG, a model is trained on multiple domains and tested on an unseen target domain. It~\cite{ben2006analysis} has demonstrated that feature representations are general and transferable to different domains if they remain invariant across domains. Motivated by this theory, a plethora of algorithms~\cite{li2017deeper,li2018deep,hu2020domain,ilse2020diva,Chuang0J20} have been proposed to learn the domain-invariant features across the source domains. However, such domain-invariant features may not be sufficient to target generalization. \cite{zhao2019learning} theoretically proved that if the marginal label distributions are significantly different between the source and target domains, the domain-invariant representation will degrade the generalization. In addition, \cite{johansson2019support} demonstrated that the requirement of domain invariance can often be excessively strict and may not always result in consistent estimates.

On the other hand, domain-specific information becomes increasingly popular for aiding the generalization ability. \cite{DingF18} employed multiple domain-specific neural networks and then aligned them together with low-rank constraints. However, too many domain-specific networks make it hard to scale to a large number of source domains. Similarly,~\cite{chattopadhyay2020learning} generated several masks within the network and each mask corresponds to one domain in the training process and then average the prediction obtained from all the individual source domain masks at test time. The main problem is that overlapping penalty is not enough for obtaining domain-specific information. More recently,~\cite{bui2021exploiting} disentangled the latent features in domain-specific and domain-invariant by minimizing the covariance matrix and meta-learning. However, it has a high computational complexity, which contains two encoders and three classifiers and involves covariance matrix calculation in high dimensionality. In addition, previous methods only considered the sufficiency of domain-specific or domain-invariant but ignored removing redundant information.  

In this paper, we propose an \textbf{I}nformation theory i\textbf{N}spired di\textbf{S}entanglement and p\textbf{UR}ification  mod\textbf{E}l (INSURE) to explicitly disentangle the latent features $\mathbf{z}$ to obtain sufficient and compact (necessary) class-relevant feature $\mathbf{z}^*$ for generalization to the unseen domain and an auxiliary feature $\mathbf{z}'$. Spesifically, inspired by information theory, we design a loss function that minimizes the KL divergence between the original feature $\mathbf{z}$ and the disentangled one $\mathbf{z}^*$ to ensure $\mathbf{z}^*$ with sufficient class relevant information. To discard the superfluous domain-specific information from $\mathbf{z}^*$, $\mathbf{z}'$ is learned to contain sufficient domain information by using a similar information theory inspired loss function. To our best knowledge, such loss functions have not been used in previous DG works. We further propose a paired purification loss function to let $\mathbf{z}'$ get rid of all the class relevant information, and therefore to ensure $\mathbf{z}^*$ contains sufficient and necessary (compact) class-relevant information. Instead of using multiple encoders, we propose to use a learnable binary mask as our disentangler to make the disentanglement more efficient and let $\mathbf{z}^*$ and $\mathbf{z}'$ complementary. It is worth mentioning that all of our loss terms are derived by theoretical analysis of the eventual goal (\ie, disentangling the sufficient and compact class relevant features). These loss terms therefore naturally align with our framework and are complementary to each other. 
Our contributions in this work are summarized as follows:
\begin{itemize}
\item We explicitly disentangle the latent features $\mathbf{z}$ to obtain sufficient and compact (necessary) class-relevant feature $\mathbf{z}^*$ for generalization to the unseen domain. We proposed an information theory inspired loss function to ensure $\mathbf{z}^*$ contains sufficient class label information and $\mathbf{z}'$ contains sufficient domain information.
\item We propose a paired purification loss function to let $\mathbf{z}'$ get rid of all the label relevant information, and thus to ensure $\mathbf{z}^*$ contains sufficient and necessary (compact) class-relevant information.
\item Instead of using multiple encoders, we use a learnable binary mask as our disentangler to make the disentanglement more efficient than traditional multiple encoders, and make $\mathbf{z}^*$ and $\mathbf{z}'$ naturally complementary.
\item We conduct extensive experiments on four widely used DG datasets, the proposed INSURE outperforms the state-of-art methods. We further empirically show that domain-specific class-relevant features are beneficial for domain generalization.
\end{itemize}

\section{Related Work}
\label{sec:rela}
\noindent\textbf{Domain Generalization.} There are a large number DG models~\cite{zhou2022domain,wang2022generalizing}, which can be broadly categorized into the following groups: (1) Domain alignment. These methods force the latent representations to have similar distribution across different domains~\cite{sun2016deep,li2018domain,jin2020feature,wang2022contrastive}. (2) Data Augmentation. DG can also be improved by data augmentation. Various techniques utilize different augmentations to simulate the unseen test domain conditions, including domain randomization~\cite{khirodkar2019domain,tobin2017domain,honarvar2020domain,huang2021fsdr}, adversarial data augmentation~\cite{volpi2018generalizing,zhao2020maximum,yang2021adversarial} and data/feature generation~\cite{somavarapu2020frustratingly,shu2021open,xu2021fourier,qiao2021uncertainty,kang2022style,zhou2021domain,nuriel2021permuted,tang2021crossnorm,xia2023generative}. (3) Learning strategy. Several learning strategies including ensemble learning~\cite{zhou2021domaintip, segu2022batch} and meta-learning~\cite{li2018deep} also improve the domain generalization. (4) Disentangled representation Learning. The goal of disentangled representation learning is to decompose a feature representation into understandable compositions (\ie, domain-invariant and domain-specific). \cite{cai2019learning} disentangles latent features in semantic and domain factors to improve performance in domain adaptation. Similarly, \cite{zhang2022towards} jointly learns the semantic and variation encoders for disentanglement and inference based on the invariant semantic features. \cite{nam2021reducing} proposes the style-agnostic networks to disentangle the style from the class categories and reduce the intrinsic style. ~\cite{peng2019domain} disentangles latent features into three parts by using three encoders as the disentangler with reconstruction losses, therefore it is much more complicated compared to our single binary mask disentangler. It disentangles only the domain invariant, class-relevant feature for inference. Moreover,  it disentangles features through adversarial learning and minimizes mutual information between the disentangled features. However, it cannot ensure the desired characteristics of the disentangled features, while our model ensures such characteristics by using information theory. Most recently, \cite{bui2021exploiting} disentangles the latent feature in domain-invariant and domain-specific and makes the final decision based on their concatenation. While our INSURE model falls under the disentangled representation learning, we emphasize our contributions focus on \textbf{what} and \textbf{how} to disentangle. In particular, we \textbf{(1)} obtain sufficient and compact (necessary) class-relevant feature $\mathbf{z}^*$ with the help of an auxiliary feature $\mathbf{z}'$, and \textbf{(2)} use the information theory inspired disentanglement and purification loss functions in a unified framework. Moreover, compared to previous works that typically employ two encoders as disentanglers, INSURE model utilizes a learnable binary mask to disentangle latent features. 
DG techniques have also found applications in various scenarios, including but not limited to few-shot learning~\cite{zhao2023fs}, hyperspectral image classification~\cite{zhang2023single}, and person re-identification~\cite{lin2020multi}.

\noindent\textbf{Information-theoretic learning for DG.} Recently, the information theory-based approaches have been widely used in the domain generalization. \cite{ahuja2021invariance} claims that invariance principle alone is insufficient and incorporating information bottleneck~\cite{tishby2000information} with Invariant Risk Minimization (IRM)~\cite{arjovsky2019invariant} improves the generalization. \cite{li2022invariant} combines information bottleneck and conditional mutual information term to achieve invariant causal prediction. Similarly, \cite{du2020learning} introduces a meta variational information bottleneck to capture the domain-invariant representation. \cite{wang2021learning} focuses on the single domain generalization and synthesizes images from diverse distribution by minimizing the mutual information between source and generated images and maximizing the mutual information among samples belonging to the same category. Most recently, \cite{chuah2022itsa} proposes an information-theoretic approach to improve the generalizability on unseen real data scenarios, which  leverages the robust information bottleneck principle~\cite{pensia2020extracting} parameterized by the statistical Fisher information. Instead of only utilizing the information bottleneck principle in the latent feature, our method further disentangles the latent features with additional mutual information and paired purification to guarantee that the class-relevant features contain and only contain class related information.

\noindent\textbf{Learnable mask.} 
\cite{mallya2018piggyback} involves learning masks for multi-task learning, the corresponding task-specific network is obtained by applying the learned masks to the backbone network. 
In~\cite{chattopadhyay2020learning}, the authors introduced the domain-specific masks to achieve the balance between specificity and invariance for domain generalization. \cite{lv2022causality} built a neural-network-based adversarial mask module to remove the inferior dimensions with less causal information. The learnable binary mask in our framework is to disentangle the class-relevant and class-irrelevant features, and it is deterministic with the sigmoid operation on the random variable instead of sampled from the Bernoulli distribution.

\section{Preliminaries}\label{sec:pre}
\subsection{Problem setting and definitions}
Let $\mathcal{X}\subset \mathbb{R}^d$ be the input space and $\mathcal{Y}\subset \mathbb{R}$ the target class label space. A domain is composed of data sampled from a joint distribution $P_{XY}$ on $\mathcal{X} \times \mathcal{Y}$. In the context of domain generalization, we are given $N$ source domains $\mathcal{S}_{\rm source}=\{\mathcal{S}^i=\{(\mathbf{x}^i,y^i)\}_{i=1}^N$, and each domain $\mathcal{S}^i$ associated with a joint distribution $P_{XY}^i$, where $(\mathbf{x}^i_j,y^i_j)\sim P_{XY}^i$. Note that the joint distribution between each pair of domains is different. A typical domain generalization framework is to learn a generalizable predictive function $C:\mathcal{X}\rightarrow\mathcal{Y}$ from the $N$ source domains and achieves a minimum prediction error on an unseen test domain $\mathcal{S}_{\rm target}$. 

We consider a learning model composed of a feature extractor $E: \mathcal{X}\rightarrow\mathcal{Z}$, where $\mathcal{Z}$ is a feature embedding space and a classifier $F: \mathcal{Z}\rightarrow\mathcal{Y}$. We divide the latent feature space $\mathcal{Z}$ into four different parts based on their association with the domains and label, the Venn diagram is illustrated in Figure \ref{figure_venn} (a). For simplicity, we consider the case with two source domains $S^1$ and $S^2$, therefore $Z_1$ and $Z_2$ are the corresponding latent features. Let $I(\cdot; \cdot)$ indicate the mutual information of two variables and $H(\cdot|\cdot)$ indicate the conditional entropy. We present the definitions of these four parts as follows:

\begin{figure}[!t]
	\centering
\begin{tabular}{c@{\hspace{0.1in}}c}
		\includegraphics[height=3.8cm]{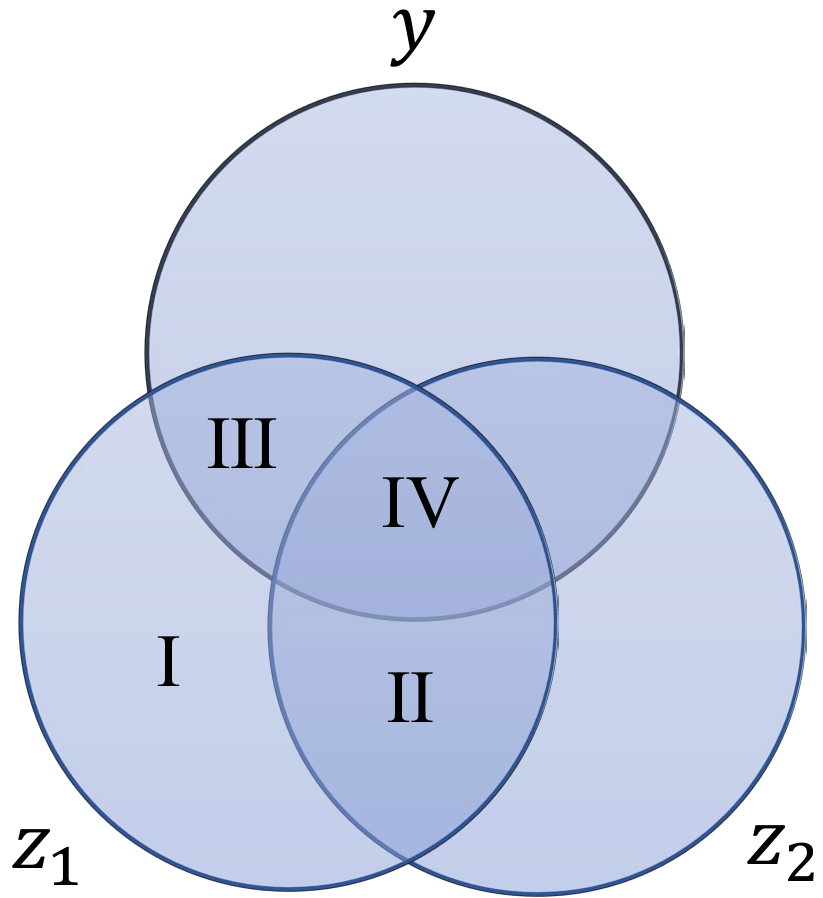}  & \includegraphics[height=3.6cm]{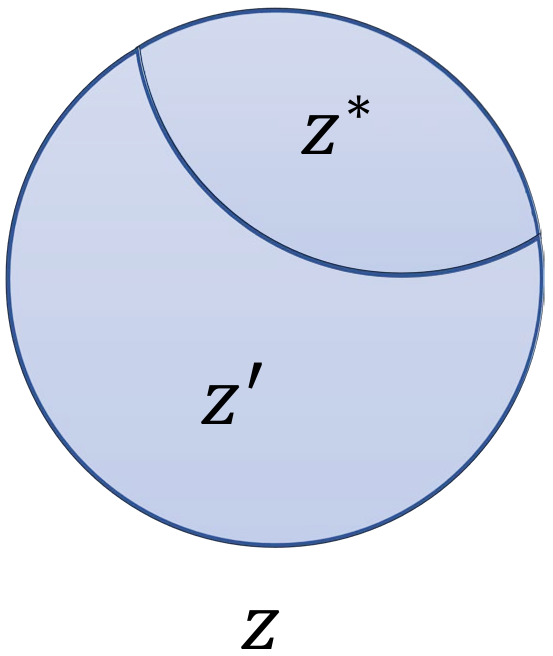} \\
		{(a) } & {(b) } 
	\end{tabular}
	\caption{\small \small (a) Venn diagram showing the relationships between two different source domains representation by $\mathbf{z}_1$, $\mathbf{z}_2$ and class label $\mathbf{y}$. Region I: \textbf{Domain-specific class-irrelevant}. Region II: \textbf{Domain-invariant class-irrelevant.} Region III: \textbf{Domain-specific class-relevant.} Region IV: \textbf{Domain-invariant class-relevant.} (b) Disentangling latent representation $\mathbf{z}$ as class-relevant $\mathbf{z^*}$ and class-irrelevant $\mathbf{z'}$.}
	\label{figure_venn}
\end{figure}

   

\begin{definition} \textit{It is said to be \textbf{\emph{Domain-Specific Class-Relevant}} for a feature extraction mapping $E: \mathcal{X}\rightarrow\mathcal{Z}$ if $\forall i.j=1,...,N,\ i\neq j$ such that $P^i(E(X))\neq P^j(E(X))$ \textbf{and} $H(Z_1|Y,Z_2)=0 \ I(Y;Z_1|Z_2)>0$. Corresponding \textbf{region \emph{III}} in Figure \ref{figure_venn} (a).}
\end{definition}

\begin{definition} \textit{It is said to be \textbf{\emph{Domain-Invariant Class-Relevant}} for a feature extractor mapping $E: \mathcal{X}\rightarrow\mathcal{Z}$ if $\forall i.j=1,...,N,\ i\neq j$ such that $P^i(E(X))\equiv P^j(E(X))$ \textbf{and} $I(Y;Z_1,Z_2)>0\ I(Z_1;Z_2|Y)=0$. Corresponding \textbf{region \emph{IV}} in Figure \ref{figure_venn} (a).}
\end{definition}

\begin{definition} \textit{It is said to be \textbf{\emph{Domain-Specific Class-Irrelevant}} for a feature extractor mapping $E: \mathcal{X}\rightarrow\mathcal{Z}$ if $\forall i.j=1,...,N,\ i\neq j$ such that $P^i(E(X))\neq P^j(E(X))$ \textbf{and} $H(Z_1|Y,Z_2)>0 \ I(Y;Z_1|Z_2)=0$. Corresponding \textbf{region \emph{I}} in Figure \ref{figure_venn} (a).}
\end{definition}

\begin{definition} \textit{It is said to be \textbf{\emph{Domain-Invariant Class-Irrelevant}} for a feature extractor mapping $E: \mathcal{X}\rightarrow\mathcal{Z}$ if $\forall i.j=1,...,N,\ i\neq j$ such that $P^i(E(X))\equiv P^j(E(X))$ \textbf{and} $I(Z_1;Z_2|Y)>0 \ I(Y;Z_1,Z_2)=0 $. Corresponding \textbf{region \emph{II}} in Figure \ref{figure_venn} (a).}
\end{definition}

Previous work typically first disentangles the latent feature into domain-specific (\textbf{\textit{region }}$\mathbf{I}$+$\mathbf{III}$) and domain-invariant (\textbf{\textit{region }}$\mathbf{II}$+$\mathbf{IV}$), and then further learn the domain-invariant class-relevant feature (\textbf{\textit{region }}$\mathbf{IV}$) by involving the class label information. However, it is argued that there still exists class-relevant information from the domain-specific part (\textbf{\textit{region }}$\mathbf{III}$), which could improve the generalizability for the unseen target domain. To capture the whole class-relevant information, we aim to disentangle the latent feature $\mathbf{z}$ in class-relevant $\mathbf{z}^*$ (\textbf{\textit{region }}$\mathbf{III}$+$\mathbf{IV}$) and class-irrelevant $\mathbf{z}'$ (\textbf{\textit{region }}$\mathbf{I}$+$\mathbf{II}$) as shown in  Figure \ref{figure_venn} (b). Therefore, the raising question is how to disentangle these two parts effectively. To answer this question, we introduce our proposed framework in the next section.

\section{Proposed Method}\label{sec:propose}

\begin{figure*}[!t]
	\centering
\begin{tabular}{c@{\hspace{0.25in}}c}
		\includegraphics[height=4.4cm]{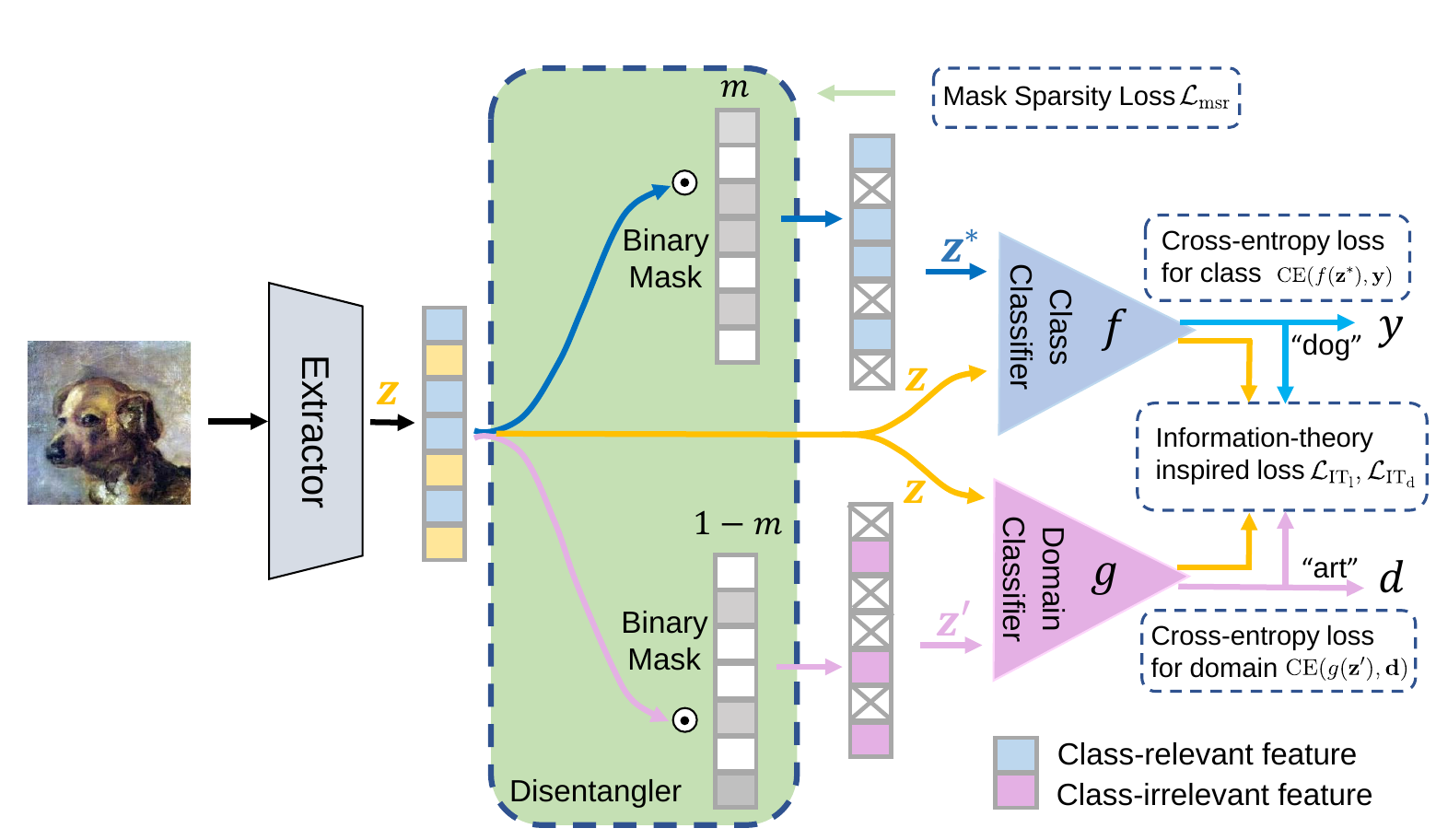} \hfill & \includegraphics[height=4.0cm]{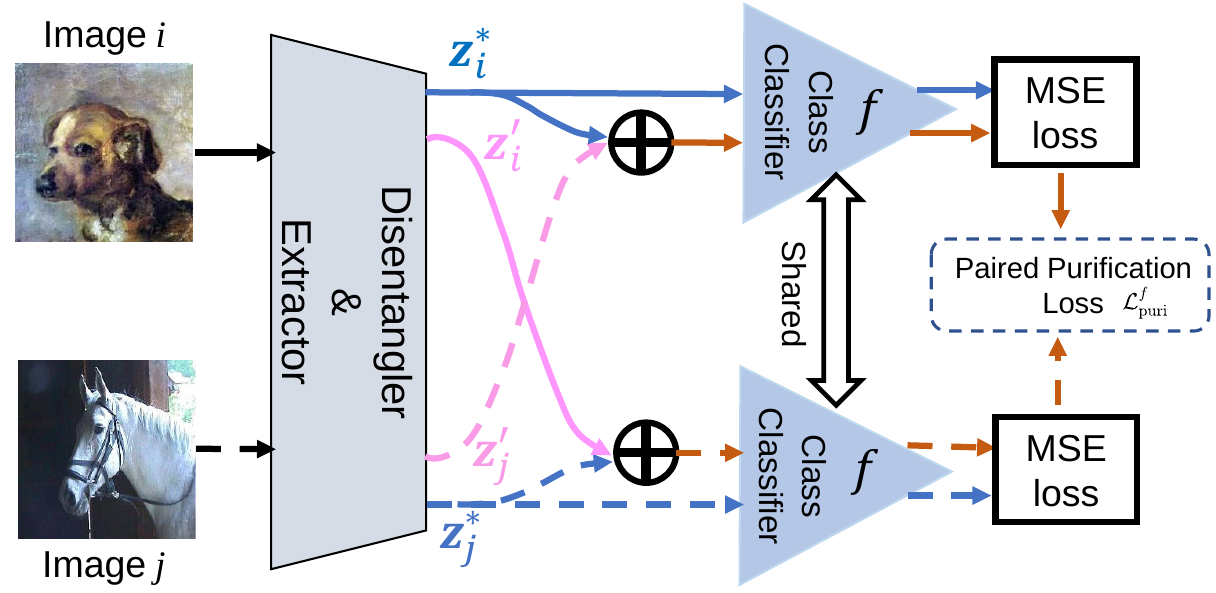} \\
		{(a) Information theory based feature disentangling. } & {(b) Paired purification. } \\
	\end{tabular}
	\caption{\small The framework of the proposed INSURE, our model consists of a feature extractor, a learnable binary mask, and two classifiers (\ie, class classifier and domain classifier). It contains two key components, namely (a) information theory-based feature disentangling and (b) paired purification. The binary mask disentangler is jointly trained with other components. The information theory-based feature disentangle ensures the $\mathbf{z}^*$ containing sufficient information related to class and paired purification further removes the redundant information within $\mathbf{z}^*$. }
	\label{fig:framework}
\end{figure*}

In this section, we will describe in detail how to learn the sufficient and necessary (compact) class-relevant feature $\mathbf{z}^*$, with the help of an auxiliary feature $\mathbf{z}'$ which will be eventually class-irrelevant. The entire framework is illustrated in Figure~\ref{fig:framework}. First, as shown in Figure~\ref{fig:framework} (a), we disentangle the original latent feature to obtain the complementary features $\mathbf{z}^*$ and $\mathbf{z}'$ by utilizing a binary mask disentangler. $\mathbf{z}^*$ is guaranteed to contain sufficient label information by using information theory. To discard the superfluous domain-specific information from $\mathbf{z}^*$, $\mathbf{z}'$ is learned to contain sufficient domain information. Then, as shown in Figure~\ref{fig:framework} (b), a paired purification loss function is proposed to eliminate all the label relevant information in $\mathbf{z}'$, and therefore ensure that $\mathbf{z}^*$ contains sufficient and necessary (compact) class-relevant information.




\subsection{Disentangling the class-relevant and class-irrelevant with a binary mask}

As shown in Figure~\ref{fig:framework}, an input image $\mathbf{x}$ is first fed into the feature extractor $E$ to get the intermediate features, referred to as $\mathbf{z}$, \ie, $\mathbf{z} = E(\mathbf{x}), \mathbf{z}\in \mathbb{R}^k$, where $k$ is the feature dimension. The intermediate feature $\mathbf{z}$ entangles class-relevant/irrelevant and domain-invariant/specific information. Our goal is to extract $\mathbf{z}^*$ that contains sufficient and necessary label information out of $\mathbf{z}$. We cast the problem as disentangling $\mathbf{z}$ to obtain $\mathbf{z}^*$ as class-relevant and $\mathbf{z}'$ as class-irrelevant to fully utilize the class label and domain index information from the multiple source domains, as well as training a class classifier $f$ and a domain classifier $g$ simultaneously.

We treat the disengagement of $\mathbf{z}$ as a feature selection problem, \ie, in $\mathbf{z}$, there are some feature dimensions that are class-relevant while the remains are class-irrelevant. Therefore, we propose to apply a binary mask as our disentangler. More specifically, given the intermediate feature $\mathbf{z}=[z_1,...,z_k]^\top$, we introduce mask parameters $\mathbf{m}=[m_1,...,m_k]^\top\in \{0,1\}^k$. The class-relevant feature $\mathbf{z}^*$ and $\mathbf{z}'$ are defined as follows:
\begin{equation}
\small
\left\{\begin{aligned}
\mathbf{z}^* &=\mathbf{z}\odot\mathbf{m}\in \mathbb{R}^k, \ m_i=\mathbb{1}(\sigma(\tilde{m}_i)<0.5)\\
\mathbf{z}' &=\mathbf{z}\odot(\mathbf{1}-\mathbf{m})\in \mathbb{R}^k,
\end{aligned}\right.
\end{equation}
where $\odot$ is the element-wise multiplication, $\sigma(\cdot)$ is the sigmoid operation and $\tilde{m}_i$ is a learnable variable.


Disentangling the latent intermediate feature with a binary mask has its advantages compared to encoder-based methods in the following aspects: (1) a binary mask requires only one learnable vector instead of multiple (usually fully connected neural network based) encoders used in encoder-based disentanglement. (2) latent features divided by the binary mask are orthogonal with each other without involving additional constraints. (3) The binary mask-based disentangler naturally maintains all the information through the disentanglement, as the summation of $\mathbf{z}^*$ and $\mathbf{z}'$ equals to $\mathbf{z}$. By contrast, the encoder-based disentanglement needs an additional decoder to reconstruct the original feature $\mathbf{z}$ to avoid information loss during disentanglement. Mask parameters are jointly trained with the feature extractor $E$ and classifiers $f$ and $g$. However, one issue with the binary mask is that we cannot update the mask parameters directly using back-propagation. We thus employ straight-through estimator~\cite{bengio2013estimating} to approximate the gradient through the binary mask.

Motivated by the information bottleneck (IB) principle~\cite{tishby2000information, AlemiFD017}, we also want to compress the latent representation $\mathbf{z}$ before the disentangler, which may improve the generalization ability by discarding irrelevant distractors in the original input $\mathbf{x}$. Thus we involve the IB principle on $\mathbf{z}$: 
\begin{equation}\label{eq:IB_loss}
\small
\mathcal{L}_{\textmd{IB}}=-I(\mathbf{z};\mathbf{y})+\epsilon I(\mathbf{z};\mathbf{x}).
\end{equation}
It encourages $\mathbf{z}$ to maximize the predictive power while compressing the information from the original image $\mathbf{x}$, where $\epsilon\geq 0$ controls the compression. Based on~\cite{AlemiFD017}, $-I(\mathbf{z};\mathbf{y})$ can be approximated as the classic cross-entropy loss, which we utilize the cross-entropy loss of the class label classification for $\mathbf{z}^*$ and the cross-entropy loss of the domain index classification for $\mathbf{z}'$.  $I(\mathbf{z};\mathbf{x})$ can be minimized by its variational upper bound defined by the KL-divergence between $q(\mathbf{z}|\mathbf{x})\sim \mathcal{N}(\bm{\mu}_\theta(\mathbf{x}),\texttt{diag}(\bm{\sigma}^2_\theta(\mathbf{x}))$) and a Gaussian normal distribution $r(\mathbf{z})\sim \mathcal{N}(0,1)$. Thus, the loss function for the disentanglement can be written as:
\begin{equation}\label{disentangle_loss}
\footnotesize
\mathcal{L}_{\rm dis}=\mathcal{L}_{\textmd{CE}}(f(\mathbf{z}^*);\mathbf{y})+\mathcal{L}_{\textmd{CE}}(g(\mathbf{z}');\mathbf{d})\\+\epsilon D_{\textmd{KL}}[q(\mathbf{z}|\mathbf{x}),r(\mathbf{z})],
\end{equation}
where $\mathbf{y}$ indicates the class label, and $\mathbf{d}$ is the domain index.

\subsection{Sufficiency of \boldmath${z}^*$ and $z'$ }

Our ideal goal is to learn $\mathbf{z}^*$ that contains sufficient and necessary label information that $\mathbf{z}$ has. As the first step, we ensure the $\mathbf{z}^*$ to keep all predictive information \emph{w.r.t.} label $\mathbf{y}$, \ie, $I(\mathbf{z};\mathbf{y})=I(\mathbf{z}^*;\mathbf{y})$. However, mutual information estimation is known as a challenging problem~\cite{tian2021farewell}. In this paper, following~\cite{tian2021farewell}, we introduce a practical calculation of the above mentioned ``sufficient". 

\begin{theorem}\label{theoerm1}
Assume the latent feature $\mathbf{z}$ is sufficient to predict the label. If the KL-divergence between the prediction distribution of the $\mathbf{z}$ and $\mathbf{z}^*$ equals to 0, then $I(\mathbf{z};\mathbf{y})=I(\mathbf{z}^*;\mathbf{y})$, \ie, $\mathbf{z}^*$ is also sufficient for the label.
\end{theorem}
\begin{proof}
$D_{\textmd{KL}} [\mathbf{P}_\mathbf{z}\|\mathbf{P}_\mathbf{z^*}]=0\Longrightarrow H(\mathbf{y}|\mathbf{z})-H(\mathbf{y}|\mathbf{z}^*)=0$. Note that $I(\mathbf{z};\mathbf{y})=H(\mathbf{y})-H(\mathbf{y}|\mathbf{z})$, $I(\mathbf{z}^*;\mathbf{y})=H(\mathbf{y})-H(\mathbf{y}|\mathbf{z}^*)$ and $I(\mathbf{z}^*;\mathbf{y})-I(\mathbf{z};\mathbf{y})=H(\mathbf{y}|\mathbf{z})-H(\mathbf{y}|\mathbf{z}^*)$. Thus minimizing $I(\mathbf{z};\mathbf{y})-I(\mathbf{z}^*;\mathbf{y})$ is equivalent to minimizing $D_{\textmd{KL}} [\mathbf{P}_\mathbf{z}\|\mathbf{P}_\mathbf{z^*}$].  
\end{proof}

With Theorem~\ref{theoerm1}, we define an information theory-based loss for class-relevant feature $\mathbf{z}^*$: 
\begin{equation}\label{suffiecient_loss_f}
\small
\mathcal{L}_{\rm IT_l}=D_{\textmd{KL}} [f(\mathbf{z})\|f(\mathbf{z}^*)],
\end{equation}
where $f$ is the class label classifier.

The mutual information between $\mathbf{z}$ and $\mathbf{z}^*$ can be factorized to two terms~\cite{federici2020learning, tian2021farewell}:
\begin{equation}\label{eq:ylabel1}
\small
I(\mathbf{z};\mathbf{z}^*)=I(\mathbf{z}^*;\mathbf{y})+I(\mathbf{z};\mathbf{z}^*|\mathbf{y}),
\end{equation}
where $I(\mathbf{z};\mathbf{z}^*|\mathbf{y})$ represents
the class-irrelevant (\ie, superfluous) information in $\mathbf{z}$. If the above mentioned sufficient condition, \ie, $I(\mathbf{z};\mathbf{y})=I(\mathbf{z}^*;\mathbf{y})$, can be achieved, then the Eq.~\ref{eq:ylabel1} becomes:
\begin{equation}\label{eq:ylabel2}
\small
I(\mathbf{z};\mathbf{z}^*)=I(\mathbf{z};\mathbf{y})+I(\mathbf{z};\mathbf{z}^*|\mathbf{y})\geq{I(\mathbf{z};\mathbf{y})},
\end{equation}
where $I(\mathbf{z};\mathbf{y})$ indicates the \textbf{\textit{region }}$\mathbf{III}$+$\mathbf{IV}$. That is to say, while $\mathbf{z}^*$ contains sufficient label information, it may also contain superfluous information from \textbf{\textit{region }}$\mathbf{I}$ and/or \textbf{\textit{region }}$\mathbf{II}$. To obtain the compact (necessary) label information, we strive for discarding the superfluous information from \textbf{\textit{region }}$\mathbf{I}$ and \textbf{\textit{region }}$\mathbf{II}$. As \textbf{\textit{region }}$\mathbf{I}$ contains the domain-specific feature, we consider letting $\mathbf{z}'$ catch up all the domain-specific information and thus enforce $\mathbf{z}^*$ to drop off the information from \textbf{\textit{region }}$\mathbf{I}$.

Similar to the effort on making $\mathbf{z}^*$ to obtain sufficient label information, we ensure the $\mathbf{z}'$ to keep all domain information \emph{w.r.t} domain index by enforcing $I(\mathbf{z};\mathbf{d})=I(\mathbf{z}';\mathbf{d})$, where $\mathbf{d}$ indicates the domain index. Following Theorem~\ref{theoerm1}, we define an information theory based loss for domain relevant feature $\mathbf{z}'$:
\begin{equation}\label{suffiecient_loss_g}
\small
\mathcal{L}_{\rm IT_d}=D_{\textmd{KL}} [g(\mathbf{z})\|g(\mathbf{z}')],
\end{equation}
where $g$ is the domain classifier. By satisfying $I(\mathbf{z};\mathbf{d})=I(\mathbf{z}';\mathbf{d})$, we obtain:
\begin{equation}\label{eq:ylabel2}
\small
I(\mathbf{z};\mathbf{z}')=I(\mathbf{z};\mathbf{d})+I(\mathbf{z};\mathbf{z}'|\mathbf{d})\geq{I(\mathbf{z};\mathbf{d})},
\end{equation}
where $I(\mathbf{z};\mathbf{d})$ indicates the \textbf{\textit{region }}$\mathbf{I}$+$\mathbf{III}$.

\subsection{Purification of \boldmath${z}^*$ }

As discussed above, $\mathbf{z}^*$ and $\mathbf{z}'$ features contain sufficient information related to the class label and domain index, respectively. The KL-divergence loss in Eq.~\ref{suffiecient_loss_f} enforces $\mathbf{z}^*$ containing \textbf{\textit{region }}$\mathbf{IV}$, and thus the $\mathbf{z}'$ does not contain any information from \textbf{\textit{region }}$\mathbf{IV}$. However, there are still gaps to achieve our ultimate goal, \ie, ensure the feature $\mathbf{z}^*$ containing sufficient and necessary label information, which is equivalent to letting the feature $\mathbf{z}^*$ contain and only contain \textbf{\textit{region }}$\mathbf{III}$+$\mathbf{IV}$. First, there is no constraint for assigning \textbf{\textit{region }}$\mathbf{II}$, and thus where it is allocated is unclear. Second, from the KL-divergence losses in Eq.~\ref{suffiecient_loss_f} and Eq.~\ref{suffiecient_loss_g}, both $\mathbf{z}^*$ and $\mathbf{z}'$ compete for \textbf{\textit{region }}$\mathbf{III}$, and thus there no guarantee that $\mathbf{z}^*$ contains whole \textbf{\textit{region }}$\mathbf{III}$. 

To fill the first gap, \ie, getting rid of \textbf{\textit{region }}$\mathbf{II}$ from $\mathbf{z}^*$, we propose to apply the mask sparsity regularization (MSR), which can be defined by the $l_1$ norm of the mask parameter vector as:
\begin{equation}\label{mask_loss}
\small
\mathcal{L}_{\rm msr}=\sum_{i=1}^{k}(1-\sigma(\tilde{m}_i)),
\end{equation}
where $k$ is the dimension of $\mathbf{z}$, and $\sigma(\cdot)$ refers the sigmoid operation. This loss encourages $\mathbf{z}^*$ to contain as less information as possible by \textit{turning on} a small number of elements within latent feature $\mathbf{z}$. The motivation is that keeping or removing a feature that is relevant to \textbf{\textit{region }}$\mathbf{II}$, referring as the domain-invariant class-irrelevant feature, will not impact any of the $L_{\rm dis}$, $L_{\rm IT_l}$, $L_{\rm IT_d}$ and $L_{\rm puri}$ loss functions, while removing this feature will decrease the msr loss and thus decrease the total loss, compared to keeping it. 

To fill the second gap, \ie, let \textbf{\textit{region }}$\mathbf{III}$ allocating to $\mathbf{z}^*$, we propose a purification strategy to prevent $\mathbf{z}'$ from containing any class-relevant information.

\begin{assumption}  
\label{Proposition 1} Denote $\mathbf{z}^*_i$ as the class-relevant features of input $\mathbf{x}_i$ and $\mathbf{z}'_j$ as the class-irrelevant features of any other input $\mathbf{x}_j$. We assume that class distribution is invariant from the variation of the class-irrelevant shift. Thus the following invariance condition should hold:
\begin{equation}\label{pair_condition}
\small
P(Y=y|Z=\mathbf{z}_i^*)=P(Y=y|Z= \mathbf{z}_i^*+\mathbf{z}'_j).
\end{equation}
\end{assumption}

This assumption shows that given a feature that is combined by the $\mathbf{z}^*$ of one sample (say $i$) and the $\mathbf{z}'$ of another sample (say $j$), its class label prediction depends only on $\mathbf{z}^*$ part, regardless of the variation of $\mathbf{z}'$. This is to say, $\mathbf{z}'$ feature does not contain any label information (from \textbf{\textit{region }}$\mathbf{III}$), and if it does,  Eq.~\ref{pair_condition} will not be satisfied. 

Based on this assumption, we propose to use the distance between $P(Y=y|Z=\mathbf{z}_i^*)$ and $P(Y=y|Z= \mathbf{z}_i^*+\mathbf{z}'_j$) as the paired purification loss function, defined as
\begin{equation}\label{eq:pair_loss_f}
\footnotesize
\mathcal{L}^{f}_{\rm puri} = \frac{1}{N}\sum_{i\neq j}^{N} \|f(\mathbf{z}_i^*),f(\mathbf{z}_i^*+\mathbf{z}'_j)\|+\|f(\mathbf{z}_j^*),f(\mathbf{z}_j^*+\mathbf{z}'_i)\|,
\end{equation}
where $\|\cdot, \cdot\|$ indicates the mean squared error (MSE) loss, $f$ refers to the class classifier, and $N$ is the number of sample pairs.

In summary, Eq.~\ref{mask_loss} helps $\mathbf{z}^*$ to discard \textbf{\textit{region }}$\mathbf{II}$ as much as possible, while Eq.~\ref{eq:pair_loss_f} requires  $\mathbf{z}'$ to get rid of \textbf{\textit{region }}$\mathbf{III}$ and thus \textbf{\textit{region }}$\mathbf{III}$ will be allocated to $\mathbf{z}^*$ thanks to the property of the binary mask.

\subsection{Training and inference}

In the training stage, we jointly train the feature extractor, the learnable binary mask and two classifiers. Our final loss function is:
\begin{equation}\label{eq:total_loss}
\small
\mathcal{L}=\mathcal{L}_{\rm dis} + \alpha\big(\mathcal{L}_{\rm IT_l} + \mathcal{L}_{\rm IT_d}\big) + \beta\mathcal{L}^{f}_{\rm puri} + \gamma\mathcal{L}_{\rm msr},
\end{equation}
where $\mathcal{L}_{\rm IT_l}$ and $\mathcal{L}_{\rm IT_d}$ encourage $\mathbf{z}^*$ and $\mathbf{z}'$ to contain sufficient information, $\mathcal{L}^f_{\rm puri}$ and $\mathcal{L}_{\rm msr}$ work together to further purify $\mathbf{z}^*$. $\alpha$, $\beta$, and $\gamma$ are selected as the balance parameters to adjust the importance of each component. In the inference stage, we only utilize $\mathbf{z}^*$ and obtain the final prediction based on the class classifier $f$.

\section{Experiments}
\label{sec:exper}
This section illustrates the superiority of our method with respect to four widely used DG benchmark datasets. Additionally, we carry out detailed ablation studies to determine the impacts of different components.
\subsection{Experimental settings}

\noindent\textbf{Dataset.} The performance of our model is evaluated on four popular datasets, including PACS, OfficeHome, TerraIncognita and DomainNet. PACS~\cite{li2017deeper} contains overall 9991 images of 7 categories from 4 domains: photo (P), art-painting (A), cartoon (C) and sketch (S). OfficeHome~\cite{venkateswara2017deep} contains 15,579 images in total with 65 categories from 4 domains of styles: Artistic (A), Clip-Art (C),
Product (P) and Real-World (R).
TerraIncognita~\cite{beery2018recognition} contains 24788 images with 10 categories from 
4 domains. DomainNet~\cite{peng2019moment} is a more recent and the largest dataset used in domain generalization tasks. It contains 0.6 million images in total with 345 categories from 6 domains: clipart, infograph, painting, quickdraw, real and sketch.


\noindent\textbf{Implementation details.}
In all of our experiments, we use the open-source code Domainbed~\cite{gulrajani2020search}. We do not apply any model selection or early stop strategy for simplicity, but just use the last model that is fully trained after all steps for the inference. We utilize ResNet-50~\cite{he2016deep} pre-trained on Imagenet as our initialization for training in all the experiments and our model is optimized with Adam optimizer~\cite{kingma2014adam}. To make the prediction more stable, followed by~\cite{arpit2021ensemble}, our model is updated with a simple moving average (SMA), starting at 100 iterations until the end of training. Both the classifiers $f$ and $g$ are one-layer MLPs. For the learnable binary mask, we initialize it as 1, which means all the neurons are turned on at the beginning. Following the literature, we train our model with 5000 iterations on PACS, OfficeHome and TerraIncognita datasets, and 20000 iterations on the DomainNet dataset, while the batch-size is set to 32 for all four datasets. We simply set the weights of each loss term in Eq.~\ref{eq:total_loss} as $\alpha=9$, $\beta=1$, and $\gamma=1$. During the training, we fix $\gamma$ while $\alpha$ and $\beta$ are slowly increasing to their final value with an exponential schedule, since starting with a larger value results in the encoder collapsing into a fixed value. 
We use Adam~\cite{kingma2014adam} optimizer for training and set the learning rate to $3.5e-4$ for the mask parameters and $5e-5$ for the remaining architectures. The weight of the information bottleneck is set to $1e-7$ for PACS, and $1e-5$ for OfficeHome, TerraIncognita, and DomainNet datasets. For the single-source domain generalization task, we remove the information bottleneck term because it harms the performance and set the learning rate to $5e-3$ for the mask parameters and $5e-5$ for the remaining architectures. The weight parameters are selected as $\alpha=10, \beta=1,\gamma=1$. All the experiments are conducted with two NVIDIA V100 GPUs, Python 3.8.13, PyTorch 1.8.0, Torchvision 0.9.0, and CUDA 11.1.

\subsection{Main results}
We evaluate INSURE model and compare it to the state of the art approaches on four standard benchmark datasets, following the settings of~\cite{gulrajani2020search,bui2021exploiting}. We illustrate the results in Table~\ref{tab:main_results}. It shows that, compared to the strong baseline (\ie, ERM); meta-learning (\ie, MLDG~\cite{li2018learning}); data augmentation (\ie, Mixup~\cite{yan2020improve, xu2020adversarial, wang2020heterogeneous}, SagNet~\cite{nam2021reducing}, RSC~\cite{huang2020self}, and FACT~\cite{xu2021fourier}), we consistently achieve the best performance. Our method also outperforms the traditional domain-specific learning (\ie, GDRO~\cite{sagawa2019distributionally}, MTL\cite{blanchard2021domain}, ARM~\cite{zhang2020adaptive}) and domain-invariant learning methods (\ie, IRM~\cite{arjovsky2019invariant}, CORAL~\cite{sun2016deep}, MMD~\cite{li2018domain}, DANN~\cite{ganin2016domain}, CDANN~\cite{li2018deep}, VREx~\cite{krueger2021out}), suggesting that focusing only on domain-invariant or domain-specific is insufficient for domain generalization. INSURE performs better than IIB~\cite{li2022invariant}, which achieves invariant causal prediction based on the information theory, because we further disentangle the latent features and discard more redundant information. In addition, our performance outperforms mDSDI~\cite{bui2021exploiting}, which disentangles latent features in domain-specific and domain-invariant parts and then inferences based on their concatenation. This demonstrates that our disentanglement is more effective. 
Note that we did not reproduce the comparison methods and all the accuracy numbers are from previous papers. From all the above comparisons, we can demonstrate the efficacy of our approach and further reveal that our class-relevant features provide more useful information and are beneficial for domain generalization.

\begin{table}[ht]
\caption{\label{tab:main_results} Comparison with state-of-art methods on PACS, OfficeHome, TerraIncognita (TI), and DomainNet with ResNet-50 ImageNet pre-trained model. The best accuracy is highlighted.}
\centering
\begin{adjustbox}{width=0.48\textwidth}
\begin{tabular}{l|cccc|c}
\hline \hline
\textbf{Model}   & \textbf{PACS}  & \textbf{OfficeHome}   & \textbf{TI}  & \textbf{DomainNet}  & \textbf{Avg}\\
\hline 
ERM~\cite{vapnik1999}                     & 85.5     & 66.5      & 46.1    & 41.3    &59.9 \\
IRM~\cite{arjovsky2019invariant}          & 83.5     & 64.3      & 47.6    & 28.0    &55.9  \\
GDRO~\cite{sagawa2019distributionally}    & 84.4     & 66.0      & 43.2    & 33.4    & 56.8\\
Mixup~\cite{yan2020improve}               & 84.6     & 68.1     & 47.9     & 39.6       &60.1 \\
MLDG~\cite{li2018learning}                & 84.9     & 66.8       & 47.7      & 41.6   &60.3   \\
CORAL~\cite{sun2016deep}                  & 86.2     & 68.7     & 47.6      & 41.8   &61.1  \\
MMD~\cite{li2018domain}                   & 84.6     & 66.3       & 42.2     & 23.5   &54.2   \\
DANN~\cite{ganin2016domain}               & 83.6     & 65.9      & 46.7  & 38.3     & 58.6   \\
CDANN~\cite{li2018deep}                   & 82.6     & 65.8       & 45.8   & 38.5  &58.2  \\
MTL\cite{blanchard2021domain}             & 84.6     & 66.4       & 45.6      & 40.8      &59.4   \\
SagNet~\cite{nam2021reducing}             & 86.3     & 68.1        & 48.6      & 40.8   &61.0   \\
ARM~\cite{zhang2020adaptive}              & 85.1     & 64.8     & 45.5     &36.0   &57.9   \\   
VREx~\cite{krueger2021out}                & 84.9     & 66.4      & 46.4    & 30.1      &57.0   \\
RSC~\cite{huang2020self}                  & 85.2     & 65.5       & 46.6       & 38.9    &59.1     \\
DMG~\cite{ChattopadhyayBH20}                 & 83.4    & -     & -     &43.6  &-\\
mDSDI~\cite{bui2021exploiting}            & 86.2     & 69.2      & 48.1     & 42.8   &61.6    \\
FACT~\cite{xu2021fourier}                 & 88.2     & 66.6      & -     & -   &-\\
IIB~\cite{li2022invariant}                & 83.9     & 68.6      & 45.8 
   & 41.5   & 60.0\\
SWAD~\cite{cha2021swad}                   & 88.2     & 70.6  & 50.0  & 46.5  &63.8    \\
PCL~\cite{yao2022pcl}                     & 88.7     & 71.6  & 52.1  & 47.7 & 65.0  \\
\hline
INSURE                                    & \textbf{89.3}     & \textbf{72.0}  & \textbf{53.1}   & \textbf{48.0} & \textbf{65.6}\\
\hline \hline
\end{tabular}
\end{adjustbox}
\end{table}

\begin{table*}[ht]
\caption{\label{tab:full_results} The full results and comparison with state-of-art methods on PACS, OfficeHome, TerraIncognita, and DomainNet with ResNet-50 ImageNet pre-trained model. The best accuracy is highlighted.}
\centering
\begin{adjustbox}{width=\textwidth}
\begin{tabular}{l|ccccc|ccccc|ccccc|ccccccc}
\hline \hline
\multirow{2}{*}{Model}&\multicolumn{5}{c|}{PACS} &\multicolumn{5}{c|}{OfficeHome} & \multicolumn{5}{c|}{TerraIncognita} & \multicolumn{7}{c}{DomainNet}\\
\cline{2-23}
& A & C & P & S &Avg & A & C & P & R & Avg & L100 & L38 & L43 & L46 & Avg & clip & info & paint & quick & real & sketch & Avg \\
\hline 
ERM~\cite{vapnik1999}&84.7&80.8&97.2&79.3&85.5& 61.3     & 52.4      & 75.8      & 76.6     & 66.5    & 49.8     & 42.1     & 56.9     & 35.7    & 46.1 & 58.6     & 19.2     & 47.0     & 13.2     & 59.9     & 49.8     & 41.3 \\
IRM~\cite{arjovsky2019invariant}& 84.8       & 76.4       & 96.7      & 76.1      & 83.5 & 58.9      & 52.2       & 72.1       & 74.0      & 64.3   & 54.6     & 39.8     & 56.2     & 39.6    & 47.6    & 40.4     & 12.1     & 31.4     & 9.8      & 37.7     & 36.7     & 28.0 \\

GDRO~\cite{sagawa2019distributionally}   & 83.5        & 79.1       & 96.7        & 78.3        & 84.4  & 60.4        & 52.7    & 75.0      & 76.0      & 66.0   & 41.2 & 38.6     & 56.7     & 36.4    & 43.2 & 47.2 & 17.5   & 34.2     & 9.2      & 51.9     & 40.1     & 33.4\\
Mixup~\cite{yan2020improve}     & 86.1      & 78.9      & 97.6       & 75.8      & 84.6 & 62.4      & 54.8      & 76.9       & 78.3       & 68.1 &\textbf{59.6}  & 42.2  & 55.9   & 33.9    & 47.9   & 55.6     & 18.7     & 45.1     & 12.8     & 57.6     & 48.2     & 39.6 \\
MLDG~\cite{li2018learning}         & 85.5        & 80.1       & 97.4       & 76.6     & 84.9 & 61.5     & 53.2       & 75.0      & 77.5   & 66.8  & 54.2     & 44.3     & 55.6     & 36.9    & 47.7   & 59.3     & 19.6     & 46.8     & 13.4     & 60.1     & 50.4     & 41.6   \\
CORAL~\cite{sun2016deep}          & 88.3        & 80.0       & 97.5       & 78.8     & 86.2 & 65.3       & 54.4     & 76.5      & 78.4   & 68.7  & 51.6     & 42.2     & 57.0     & 39.8    & 47.6   & 59.2     & 19.9     & 47.4     & 14.0     & 59.8     & 50.4     & 41.8\\
MMD~\cite{li2018domain}           & 86.1        & 79.4       & 96.6       & 76.5      & 84.6& 60.4      & 53.3       & 74.3     & 77.4    & 66.3  & 41.9     & 34.8     & 57.0     & 35.2    & 42.2   & 32.2     & 11.2     & 26.8     & 8.8      & 32.7     & 29.0     & 23.5\\
DANN~\cite{ganin2016domain}     & 86.4        & 77.4       & 97.3         & 73.5        & 83.6& 59.9     & 53.0      & 73.6     & 76.9      & 65.9   & 51.1     & 40.6     & 57.4     & 37.7    & 46.7   & 53.1     & 18.3     & 44.2     & 11.9     & 55.5     & 46.8     & 38.3 \\
CDANN~\cite{li2018deep}         & 84.6        & 75.5       & 96.8         & 73.5        & 82.6& 61.5     & 50.4       & 74.4   & 76.6  & 65.8   & 47.0     & 41.3     & 54.9     & 39.8    & 45.8   & 54.6     & 17.3     & 44.2     & 12.8     & 56.2     & 45.9     & 38.5\\
MTL\cite{blanchard2021domain}    & 87.5        & 77.1       & 96.4       & 77.3       & 84.6 & 61.5      & 52.4       & 74.9      & 76.8      & 66.4    & 49.3     & 39.6     & 55.6     & 37.8    & 45.6    & 58.0     & 19.2     & 46.2     & 12.7     & 59.9     & 49.0     & 40.8\\
SagNet~\cite{nam2021reducing}    & 87.4       & 80.7       & 97.1       & 80.0        & 86.3 & 63.4       & 54.8        & 75.8      & 78.3   & 68.1  & 53.0     & 43.0     & 57.9     & 40.4    & 48.6   & 57.7     & 19.1     & 46.3     & 13.5     & 58.9     & 49.5     & 40.8\\
ARM~\cite{zhang2020adaptive}     & 86.8       & 76.8        & 97.4       & 79.3        & 85.1& 58.9        & 51.0      & 74.1      & 75.2   & 64.8   & 49.3     & 38.3     & 55.8     & 38.7   & 45.5   & 49.6     & 16.5     & 41.5     & 10.8     & 53.5     & 43.9     & 36.0\\   
VREx~\cite{krueger2021out}     & 86.0       & 79.1       & 96.9        & 77.7        & 84.9& 60.7       & 53.0      & 75.3      & 76.6      & 66.4   & 48.2     & 41.7     & 56.8     & 38.7   & 46.4   & 43.3     & 14.1     & 32.5     & 9.8      & 43.5     & 37.7     & 30.1 \\
RSC~\cite{huang2020self}       & 85.4       & 79.7       & 97.6       & 78.2     & 85.2& 60.7      & 51.4       & 74.8       & 75.1     & 65.5   & 50.2     & 39.2     & 56.3     & 40.8    & 46.6   & 55.0     & 18.3     & 44.4     & 12.5     & 55.7     & 47.8     & 38.9 \\
mDSDI~\cite{bui2021exploiting} & 87.7  & 80.4  & \textbf{98.1}  & 78.4 & 86.2& 68.1      & 52.1      & 76.0      & 80.4    & 69.2   & 53.2      & 43.3      & 56.7      & 39.2    & 48.1     & 62.1      & 19.1      & 49.4      & 12.8    & 62.9      & 50.4   & 42.8  \\
SWAD~\cite{cha2021swad}   & 89.3  & 83.4  & 97.3  & 82.5  & 88.1& 66.1  & 57.7  & 78.4  & 80.2  & 70.6   & 55.4     & 44.9     & 59.7     & 39.9    & 50.0   & 66.0   & 22.4   & 53.5   & \textbf{16.1}   & 65.8  & 55.5  & 46.5\\
PCL~\cite{yao2022pcl}   & \textbf{90.2}  & 83.9  & \textbf{98.1}  & 82.6  & 88.7& 67.3  & \textbf{59.9}  & \textbf{78.7}  & 80.7  & 71.6   & 58.7     & 46.3     & 60.0     & 43.6    & 52.1   & \textbf{67.9}   & \textbf{24.3}   & 55.3   & 15.7   & 66.6  & 56.4  & 47.7\\
\hline
INSURE                 & \textbf{90.2} &\textbf{85.3} &97.9 &\textbf{83.8} &\textbf{89.3}& \textbf{71.4}     & 57.3     & 78.0     & \textbf{81.2}     & \textbf{72.0}  & 58.8     & \textbf{46.4}  & \textbf{61.7}  & \textbf{45.5}  & \textbf{53.1}   & 67.8   & 24.0   & \textbf{55.6}   & 16.0   & \textbf{67.6}  & \textbf{57.2}  & \textbf{48.0} \\
\hline \hline
\end{tabular}
\end{adjustbox}
\end{table*}

\subsection{Ablation study}
\noindent\textbf{Contribution of each component:} We conduct an extensive ablation study on the PACS dataset to investigate the effectiveness of each component in the INSURE model. In Table~\ref{tab:ablation_result}, the ``Baseline" model applies binary mask as the disentangler and only contains disentanglement loss $\mathcal{L}_{\rm dis}$. Adding each loss term separately to the baseline model improves the performance showing the effectiveness of each component. Specifically, we observe the accuracy of combining $\mathcal{L}_{\rm msr}$ and $\mathcal{L}_{\rm IT}$ to baseline model perform worse than only adding $\mathcal{L}_{\rm msr}$. That means, $\mathcal{L}_{\rm IT}$ only encourages sufficiency, and $\mathbf{z}^*$ still contains superfluous information, which tends to degrade the performance. 
The best performance is achieved by combining all terms together indicating that each loss works as an indispensable component in our framework. To further demonstrate the corporation of each loss term, we visualize the distributions of $\mathbf{z}^*$ in Figure~\ref{fig:t_SNE}, we can see the different classes more distinguishable, \emph{e.g.,} the distance between the person and the other categories is greater than that of other methods, when combining all the terms. In addition, we utilize the visualization technique~\cite{selvaraju2017grad} to present attention maps of the last convolutional layer in terms of $\mathbf{z}^*$ learned by our proposed method with different components in Figure \ref{fig:saliency_map}. It shows that our proposed INSURE is more capable of capturing the entire class-related information than others. Taking the elephant as an example (the first row), INSURE focuses on the whole elephant and some areas near the elephant, indicating that some domain-specific class-relevant information could also improve generalizability.

\begin{table}[htbp]
\centering
\small
\caption{\label{tab:ablation_result} Ablation study of each component in our proposed objective function.}
\begin{adjustbox}{width=\linewidth}
\begin{tabular}{l|ccccc}
\hline \hline
\textbf{Model}   & \textbf{A}   & \textbf{C}   & \textbf{P}  & \textbf{S} & \textbf{Avg}  \\
\hline
ERM~\cite{vapnik1999}& 84.7    & 80.8    & 97.2    & 79.3   & 85.5 \\  
Baseline \small{($w/$ binary mask)}       &87.6    &82.3    &96.8   &80.6 &86.8  \\
\hline
+msr      &88.0     &83.6    &97.0  &81.9 &87.6  \\
+IT       &88.2      &83.0    &96.4   &81.1&87.2     \\
+Puri      &87.5    &82.8     &97.2     &82.0 &87.4     \\
+msr+IT     &88.5     &83.1    &97.0   &81.3&87.5  \\
+msr+Puri      &89.8     &83.4   &97.0   &81.6&87.9     \\
+IT+Puri      &88.5    &83.7      &97.2   &81.9  &87.8     \\
\hline
INSURE (Full Model) &\textbf{90.2}     &\textbf{85.3}     &\textbf{97.9}  &\textbf{83.8} &  \textbf{89.3} \\
\hline \hline
\end{tabular}
\end{adjustbox}
\end{table}


\noindent\textbf{Validation of a binary mask disentangler:} We validate the effectiveness and efficiency of the proposed binary mask disentangler by comparing it with two multi-encoder based models. The first one simply replaces the binary mask disentangler with two MLP encoders in the INSURE model, while the other incorporates two feature extractors ~\cite{bui2021exploiting}. Conventional disentanglers typically incorporate extra parameters because of the multiple encoders, and considerable computational expenses owing to additional losses required to ensure disentangled features are both independent and lossless. In contrast, our binary mask disentangler is simply a learnable vector, which directly guarantees that the disentangled features are orthogonal and lossless. Table~\ref{tab:binary_mask} shows the performance, training time (per step), and the number of parameters, which demonstrates the advantages of the binary mask disentangler.


\begin{table}[ht]
\centering
\small
\caption{\label{tab:binary_mask} Comparison of different disentanglers on PACS.}
\begin{adjustbox}{width=\linewidth}
\begin{tabular}{l|ccc}
\hline \hline
\textbf{Model}   & \textbf{PACS}  & \textbf{Training time}   & \textbf{Params}  \\
\hline
mDSDI~\cite{bui2021exploiting}  &86.2& 2.18s  &55.4M\\
Ours $w/$ two-encoder  &86.9   &1.83s&38.6M\\
\hline
INSURE ($w/$ binary mask) &\textbf{89.3}& \textbf{1.02}s  &\textbf{30.3}M\\
\hline \hline
\end{tabular}
\end{adjustbox}
\end{table}

Furthermore, we evaluate the Performance of different mask types on the PACS benchmark dataset. While the binary mask is usually used in a hard manner, \ie, all elements can only be either 0 or 1, we also evaluate the effectiveness of the soft binary mask. \ie, all elements can be a continuous value between 0 and 1. The hard and soft binary masks are used during the training and inference stage, and we report the results in Table~\ref{hard_soft}. We can see that in the training stage, the hard mask outperforms the soft mask. We argue the reason is that there may exist some trivial solutions, such as all the elements of the mask being the same or very similar non-zero values which make $\mathbf{z}^*$ and $\mathbf{z}'$ highly relevant to each other and thus contain the same information but only different scales. This way, the disentanglement does not work at all. In the inference stage, the performance of using the soft or hard masks is almost equivalent. 

\begin{table}[htbp]
\centering
\caption{\label{tab:mask_result} Performance of different mask types on PACS benchmark dataset.}
\begin{tabular}{l|ccccc}
\hline \hline
\textbf{Training/Inference}   & \textbf{P}   & \textbf{A} & \textbf{C} & \textbf{S}  & \textbf{Avg}  \\
\hline
$Hard/Hard$   &90.2 &85.3 &97.9    &83.8  &89.3\\ 
$Hard/Soft$   & 90.0 &85.2 &97.8 &83.8 &89.2\\
$Soft/Hard$     &89.7  &83.5&97.0&81.7&88.0\\
$Soft/Soft$     &89.8  &83.5&97.0&81.8&88.0\\
\hline \hline
\end{tabular}
\label{hard_soft}
\end{table}

\noindent\textbf{Parameter sensitivity:} In Table~\ref{tab:Hyper-parameters}, we show the sensitivity analysis to the weight parameter of different loss terms, $\alpha$, $\beta$ and $\gamma$ in Eq.~\ref{eq:total_loss}. When we analyze the sensitivity to a specific parameter, the other two keep being selected values, \ie $\alpha=9$, $\beta=1$, and $\gamma=1$. 



\begin{figure*}
\centering
  \includegraphics[width=1\linewidth,height=36mm]{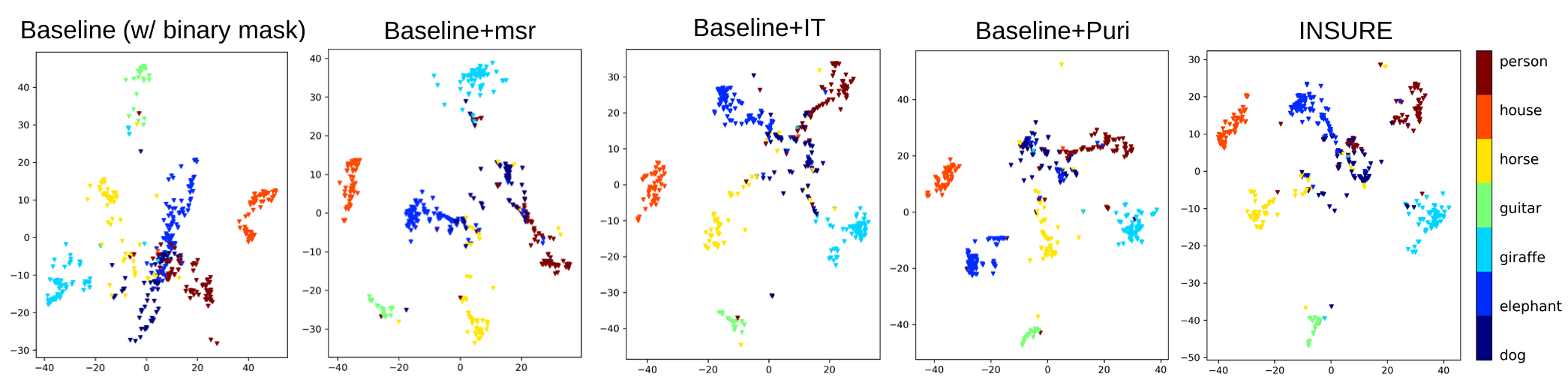}
  \caption{t-SNE~\cite{van2008visualizing} visualization on proposed method with different components. We visualize $\mathbf{z}^*$ on the PACS benchmark, where the target domain is cartoon.}
  \label{fig:t_SNE}
\end{figure*}

\begin{figure}
\centering
\includegraphics[width=\linewidth]{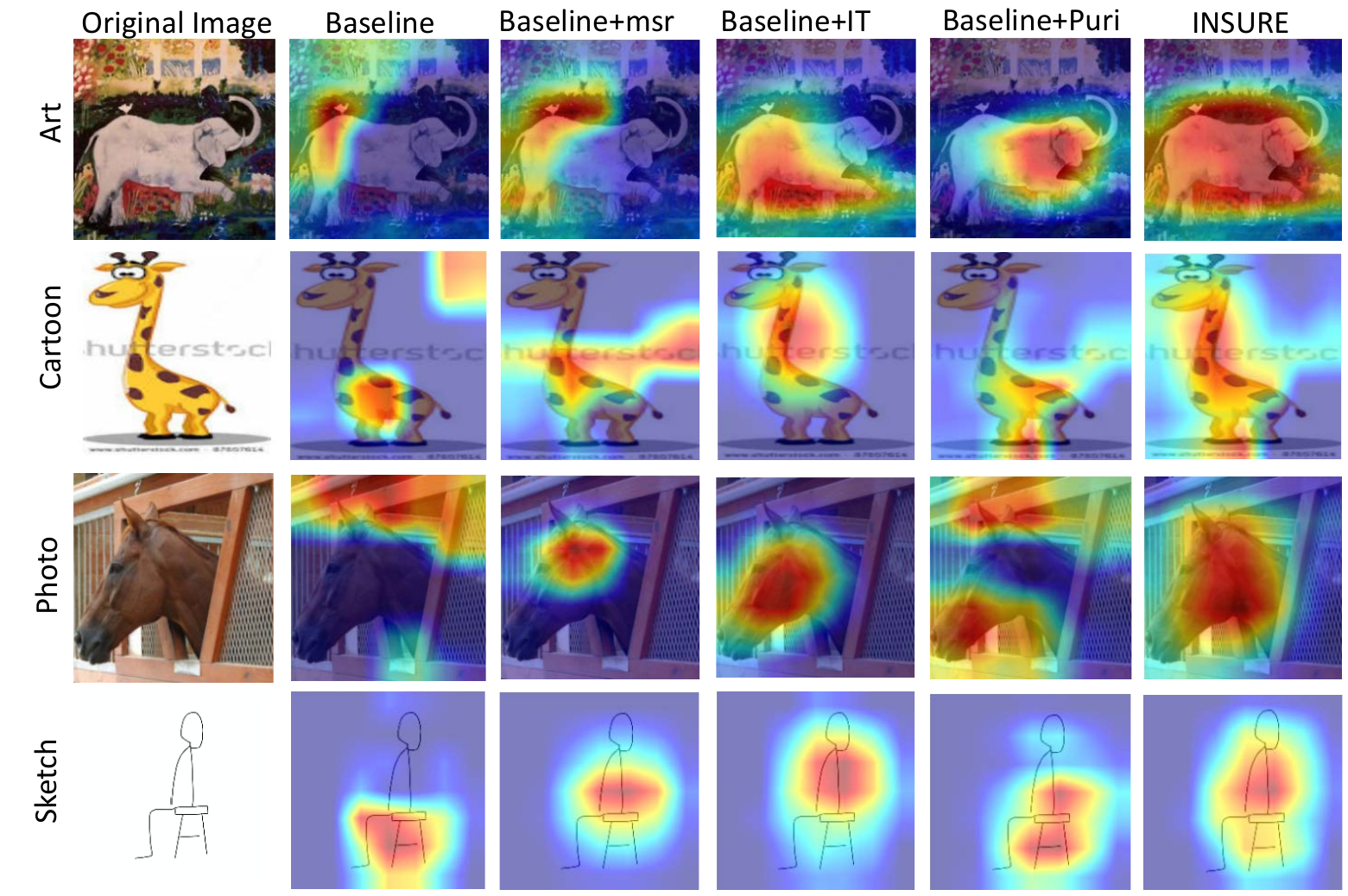}  \caption{\small Attention maps in terms of $\mathbf{z}^*$ on four different target domains of the PACS benchmark by using different components of our proposed method.}
\label{fig:saliency_map}
\end{figure}

\begin{table}[ht]
\caption{\label{tab:Hyper-parameters} Sensitivity analysis of INSURE with respect to parameters $\alpha$, $\beta$ and $\gamma$ in Eq.~\ref{eq:total_loss}. on PACS dataset.}
\centering
\begin{adjustbox}{width=1\linewidth}
\huge
\begin{tabular}{ccccc|ccccc|ccccc}
\hline \hline
\multicolumn{5}{c|}{$\alpha=$} & \multicolumn{5}{c|}{$\beta=$} & \multicolumn{5}{c}{$\gamma=$}\\
\hline
5 &7&9&12&15& 0.5 & 0.7 &1&1.2& 1.5 & 0.5 & 0.7 &1&1.2& 1.5 \\
\hline
88.0&89.0&\textbf{89.3}&88.7&88.1&  
87.6& 88.0& \textbf{89.3}& 89.0& 88.7& 
87.9& 88.7& \textbf{89.3}&88.8&88.0  \\
\hline \hline
\end{tabular}
\end{adjustbox}
\end{table}



\noindent\textbf{Different random seeds}
\begin{table*}
\large
\centering
\caption{\label{tab:full_result}The statistical results of the performance from five trials with different random seeds on PACS, OfficeHome, and TerraLincognita benchmark datasets.}
\begin{adjustbox}{width=\textwidth}
\begin{tabular}{l|cccc|cccc|cccc|cccc|cccc}
\hline \hline
\multirow{3}{*}{PACS}   & \multicolumn{4}{c|}{\textbf{A}} & \multicolumn{4}{c|}{\textbf{C}}  & \multicolumn{4}{c|}{\textbf{P}} & \multicolumn{4}{c|}{\textbf{S}} & \multicolumn{4}{c}{\textbf{Avg}}\\
& min & max  & \multicolumn{2}{c|}{mean/std.} & min & max  & \multicolumn{2}{c|}{mean/std.}  & min & max  & \multicolumn{2}{c|}{mean/std.} 
 & min & max  & \multicolumn{2}{c|}{mean/std.}   & min & max  & \multicolumn{2}{c}{mean/std.} \\
 \cline{2-21}
 & $89.1$ & $90.2$ &  \multicolumn{2}{c|}{ $89.5\pm 0.4$}  & $84.2$ & $86.4$ & \multicolumn{2}{c|}{ $85.5\pm 0.7$ }  & $97.0$ &$97.9$ & \multicolumn{2}{c|}{ $97.3\pm 0.3$ }  & $83.2$ & $85.5$ & \multicolumn{2}{c|}{$84.5\pm 0.5$}  & $88.7$ & $89.6$ & \multicolumn{2}{c}{$89.2\pm 0.3$} \\
\hline \hline
\multirow{3}{*}{OfficeHome}   & \multicolumn{4}{c|}{\textbf{A}} & \multicolumn{4}{c|}{\textbf{C}}  & \multicolumn{4}{c|}{\textbf{P}} & \multicolumn{4}{c|}{\textbf{R}} & \multicolumn{4}{c}{\textbf{Avg}}\\
& min & max  & \multicolumn{2}{c|}{mean/std.} & min & max  & \multicolumn{2}{c|}{mean/std.}  & min & max  & \multicolumn{2}{c|}{mean/std.} 
 & min & max  & \multicolumn{2}{c|}{mean/std.}   & min & max  & \multicolumn{2}{c}{mean/std.} \\
 \cline{2-21}
 & $70.7$& $71.4$ &  \multicolumn{2}{c|}{ $71.1\pm 0.3$}  & $57.1$ & $57.3$& \multicolumn{2}{c|}{ $57.2\pm 0.1$ }  & $78.0$ & $78.6$ & \multicolumn{2}{c|}{ $78.3\pm 0.2$ }  & 81.2 & $82.0$ & \multicolumn{2}{c|}{$81.6\pm 0.3$}  & $72.0$ & $72.1$ & \multicolumn{2}{c}{$72.0\pm 0.1$} \\
 \hline\hline
\multirow{3}{*}{TerraIncognita}   & \multicolumn{4}{c|}{\textbf{L100}} & \multicolumn{4}{c|}{\textbf{L38}}  & \multicolumn{4}{c|}{\textbf{L43}} & \multicolumn{4}{c|}{\textbf{L46}} & \multicolumn{4}{c}{\textbf{Avg}}\\
& min & max  & \multicolumn{2}{c|}{mean/std.} & min & max  & \multicolumn{2}{c|}{mean/std.}  & min & max  & \multicolumn{2}{c|}{mean/std.} 
 & min & max  & \multicolumn{2}{c|}{mean/std.}   & min & max  & \multicolumn{2}{c}{mean/std.} \\
 \cline{2-21}
 & $57.8$ & $58.8$ &  \multicolumn{2}{c|}{ $58.2\pm 0.2$}  & $45.8$ & $50.8$ & \multicolumn{2}{c|}{ $47.3\pm 0.8$ }  & $61.4$ & $62.2$ & \multicolumn{2}{c|}{ $61.7\pm 0.3$ }  & $44.1$ & $45.5$ & \multicolumn{2}{c|}{$45.1\pm 0.6$}  & $52.6$ & $54.1$ & \multicolumn{2}{c}{$53.1\pm 0.5$} \\
 \hline\hline
\end{tabular}
\end{adjustbox}
\end{table*}

\begin{table*}
\Huge
\centering
\caption{\label{tab:full_result_Domainnet}The statistical results of the performance from five trials with different random seeds on DomainNet benchmark dataset.}
\begin{adjustbox}{width=\textwidth}
\begin{tabular}{l|cccc|cccc|cccc|cccc|cccc|cccc|cccc}
\hline \hline
\multirow{3}{*}{DomainNet}   & \multicolumn{4}{c|}{\textbf{clip}} & \multicolumn{4}{c|}{\textbf{info}}  & \multicolumn{4}{c|}{\textbf{paint}} & \multicolumn{4}{c|}{\textbf{quick}} & \multicolumn{4}{c|}{\textbf{real}}& \multicolumn{4}{c|}{\textbf{sketch}}& \multicolumn{4}{c}{\textbf{Avg}}\\
& min & max  & \multicolumn{2}{c|}{mean/std.} & min & max  & \multicolumn{2}{c|}{mean/std.}  & min & max  & \multicolumn{2}{c|}{mean/std.} 
 & min & max  & \multicolumn{2}{c|}{mean/std.}   & min & max  & \multicolumn{2}{c|}{mean/std.}& min & max  & \multicolumn{2}{c|}{mean/std.}& min & max  & \multicolumn{2}{c}{mean/std.} \\
 \cline{2-29}
 & 67.6 & 67.8 &  \multicolumn{2}{c|}{ $67.7\pm 0.1$}  & 23.9 & 24.4 & \multicolumn{2}{c|}{ $24.1\pm 0.2$ }  & 55.5 & 55.7 & \multicolumn{2}{c|}{ $55.6\pm 0.1$ }  & 16.0 & 16.6 & \multicolumn{2}{c|}{$16.3\pm 0.3$}  & 67.5 & 67.6 & \multicolumn{2}{c|}{$67.6\pm 0.0$} & 57.2 & 57.5& \multicolumn{2}{c|}{$57.3\pm 0.1$}& 48.0 & 48.2 & \multicolumn{2}{c}{$48.1\pm 0.1$}\\

\hline \hline
\end{tabular}
\end{adjustbox}
\end{table*}

The training procedure would introduce a certain level of randomness, \eg, the way to split the training and validation set, the order of the data samples for iterations, the initialization of the class label classifier $f$ and domain index classifier $g$, etc. To keep the reproducibility, we fix the random seed to be 0 for all our experiments in the main text. Here, to investigate how our model is sensitive to randomness, we conduct repeat our experiments with random seeds for five trials. The minimal, maximal, mean and standard deviation numbers are reported in Table~\ref{tab:full_result} and Table~\ref{tab:full_result_Domainnet}. Our proposed method is not sensitive to randomness and consistently outperforms the state-of-art methods.

\subsection{Evaluation of Single Domain Generalization}
We also evaluate INSURE model in a more challenging scenario, single-source domain generalization (single-DG), where only one source domain is available for training. Since it lacks domain index information, we simply remove the domain classifier $g$ from our framework, \ie, excluding $\mathcal{L_{\rm CE}} (g(\mathbf{z}');\mathbf{d})$ and $\mathcal{L}_{\rm IT_d}$ from the loss function. Table~\ref{single_source} illustrates INSURE model outperforms two SOTA models, as well as a baseline ERM model, on PACS with ResNet-18 and DomainNet with ResNet-50. 

\begin{table}[htbp]

\centering\small
\caption{\label{single_source}Results of single-source domain generalization.}
\begin{adjustbox}{width=0.32\textwidth}
\begin{tabular}{l|cc}
\hline \hline
\textbf{Algorithm}   & \textbf{PACS}   & \textbf{DomainNet}  \\
\hline
ERM    &60.7&21.7   \\
L2D~\cite{wang2021learning}   &65.2&-\\
MetaCNN~\cite{Wan_2022_CVPR} &-&25.6\\
\hline
INSURE   &\textbf{66.2} & \textbf{26.8} \\
\hline \hline
\end{tabular}
\end{adjustbox}
\end{table}
\subsection{Effectiveness of \textbf{\textit{region }}$\mathbf{III}$}
In this section, we investigate whether the domain-specific and class-relevant features (\textbf{\textit{region }}$\mathbf{III}$) really contribute to the generalization of the unseen target domains by only adjusting the paired purification loss term in the final objective function. As discussed above, in Eq.~\ref{eq:pair_loss_f}, the proposed loss function upon class classifier $f$ make $\mathbf{z}^*$ catch all \textbf{\textit{region }}$\mathbf{III}$. Similarly, we can also make $\mathbf{z}^*$ discard all \textbf{\textit{region }}$\mathbf{III}$ by defining a paired purification loss function upon domain classifier $g$ as follows:
\begin{equation}\label{eq:pair_loss_g}
\footnotesize
\begin{split}
\mathcal{L}^{g}_{\rm puri} = \frac{1}{N}\sum_{i\neq j}^{N} \|g(\mathbf{z}'_i),g(\mathbf{z}'_i+\mathbf{z}^*_j)\|+\|g(\mathbf{z}'_j),g(\mathbf{z}'_j+\mathbf{z}^*_i)\|.
\end{split}
\end{equation}
As shown in Table~\ref{tab:dis_result}, the performance when $\mathbf{z}^*$ contains \textbf{\textit{region }}$\mathbf{III}$ outperforms that when $\mathbf{z}^*$ not contain \textbf{\textit{region }}$\mathbf{III}$, which suggests that domain-specific and class-relevant feature is effective for generalizability.


\begin{table}[htbp]
\centering
\caption{\label{tab:dis_result} Verifying the effectiveness of the domain-specific class-relevant feature (\textbf{\textit{region }}$\mathbf{III}$).}
\begin{adjustbox}{width=.9\linewidth}
\begin{tabular}{l|ccccc}
\hline \hline
\textbf{Model}   & P  & A & C & S  & \textbf{Avg}  \\
\hline
\emph{NOT} contains \textbf{\textit{region }}$\mathbf{III}$       &88.8     &83.8   &97.1  &80.8&87.6   \\
\hline
Contains \textbf{\textit{region }}$\mathbf{III}$ &90.2 &85.3 &97.9    &83.8  &\textbf{89.3}\\
\hline \hline
\end{tabular}
\end{adjustbox}
\end{table}



\section{Conclusion}
In this paper, we proposed the INSURE model to explicitly disentangle the latent features to obtain sufficient and compact (necessary) class-relevant features for domain generalization tasks. We designed a loss function based on information theory to ensure the two disentangled features contain sufficient label and domain information, respectively and further proposed a paired purification loss function to obtain the sufficient and compact (necessary) class-relevant feature. Comprehensive experiments on four DG benchmark datasets showed that our proposed model outperformed the state-of-art methods. We also empirically showed that domain-specific class-relevant feature is beneficial for domain generalization.


{\small
\bibliographystyle{IEEEtran}
\bibliography{egbib}
}

\end{document}